\documentclass{article}

\usepackage{arxiv}
\usepackage{comment}
\usepackage[utf8]{inputenc} % allow utf-8 input
\usepackage[T1]{fontenc}    % use 8-bit T1 fonts
\usepackage{hyperref}       % hyperlinks
\usepackage{url}            % simple URL typesetting
\usepackage{booktabs}       % professional-quality tables
\usepackage{amsfonts}       % blackboard math symbols
\usepackage{nicefrac}       % compact symbols for 1/2, etc.
\usepackage{microtype}      % microtypography
\usepackage{lipsum}
\usepackage{amsthm}
\usepackage{color}

\newif\ifworkinprogress
\workinprogresstrue
\ifworkinprogress
   
  \newcommand{\jw}[1]{\textcolor{blue}{\textbf{[Jiawei] #1}}}
  \newcommand{\pv}[1]{\textcolor{blue}{\small{[Puya] #1}}}
  
\else
  \newcommand{\mq}[1]{}
  \newcommand{\es}[1]{}
  \newcommand{\mr}[1]{}
  \newcommand{\ty}[1]{}
  \newcommand{\pv}[1]{}
\fi
\usepackage{csquotes}
\usepackage{epsfig}
\usepackage{color}
\newcommand{\hl}[1]{\colorbox{yellow}{#1}}
\usepackage{balance}  % for \balance command ON LAST PAGE  (only there!)
\usepackage{algorithmic}
\usepackage{algorithm}
\usepackage{multirow}
\usepackage{graphicx}
\usepackage{amsmath}
\usepackage{amssymb}
\usepackage{amsfonts}
\usepackage{xspace}
\usepackage{url}
\usepackage{./style/tweaklist}
\usepackage{hyperref}
\usepackage{booktabs}
\usepackage{bbold}
\usepackage{bm}
\usepackage{mathtools}
\DeclarePairedDelimiter\ceil{\lceil}{\rceil}

\DeclareMathOperator{\E}{\mathbb{E}}
\newcommand{\normaltilde}{\raise.17ex\hbox{$\scriptstyle\sim$}}
\def\argmin{\mbox{argmin}}
\def\argmax{\mbox{argmax}}
\newcommand{\IGNORE}[1]{}

\newcommand{\CRM}{\textit {CRM}\xspace}
\newcommand{\VALUE}{\textit {VALUE}\xspace}
\newcommand{\ZERO}{\textit {ZERO}\xspace}
\newcommand{\AVG}{\textit {AVG}\xspace}
\newcommand{\REVMAX}{\textit {REV\_MAX}\xspace}
\newcommand{\BR}{\textit {BR}\xspace}
\newcommand{\BRW}{\textit {BR\_w}\xspace}
\newcommand{\PDR}{\textit {PDR}\xspace}
\newcommand{\RECALL}{\textit {RECALL}\xspace}
\newcommand{\REVPOTENT}{\textit {REV\_POTENT}\xspace}
\newcommand{\PDRLP}{\textit {PDR\_LP}\xspace}
\newcommand{\PDRHP}{\textit {PDR\_HP}\xspace}

\newcommand{\commentedtext}[1]{}

\DeclareMathOperator*{\median}{median}

\newtheorem{corollary}{Corollary}
\newtheorem{theorem}{Theorem}
\newtheorem{lemma}{Lemma}
\graphicspath{{./figures/}}

\begin{document}
%Conference

%\title{Optimal Multidimensional Pricing for Airbnb Marketplace}

\title{Revenue Maximization of Airbnb Marketplace using Search Results}
%\titlenote{Produces the permission block, and copyright information}
%\subtitle{Extended Abstract}
%\subtitlenote{The full version of the author's guide is available as
%  \texttt{acmart.pdf} document}

\author{
  Jiawei Wen \\
  Pennsylvania State University\\
  \texttt{jiaweiwen7@gmail.com} \\
  %% examples of more authors
   \And
 Puya Vahabi \\
    Airbnb\\
  \texttt{Puya.vahabi@airbnb.com}\\
  \And
   Mihajlo Grbovic \\
   Airbnb \\
   \texttt{mihajlo.grbovic@airbnb.com} \\
  %% \And
  %% Coauthor \\
  %% Affiliation \\
  %% Address \\
  %% \texttt{email} \\
  %% \And
  %% Coauthor \\
  %% Affiliation \\
  %% Address \\
  %% \texttt{email} \\
}


\date{}
\maketitle

\begin{abstract}
Correctly pricing products or services in an online marketplace presents a challenging problem and one of the critical factors for the success of the business. When users are looking to buy an item they typically search for it. Query relevance models are used at this stage to retrieve and rank the items on the search page from most relevant to least relevant. The presented items are naturally “competing” against each other for user purchases. 
We provide a practical two-stage model to price this set of retrieved items for which distributions of their values are learned. The initial output of the pricing strategy is a price vector for the top displayed items in one search event. We later aggregate these results over searches to provide the supplier with the optimal price for each item.  
We applied our solution to large-scale search data obtained from Airbnb Experiences marketplace. Offline evaluation results show that our strategy improves upon baseline pricing strategies on key metrics by at least +$20$\% in terms of booking regret and +$55$\% in terms of revenue potential.

%by over $80\%$.

% outperforms several known marketplace pricing strategies on average by +$20\%$.

\end{abstract}


% INTRODUCTION
\section{Introduction}
\label{sec:introduction}
%!TEX root = paper.tex
%
% INTRODUCTION
% 1. What is the problem?

%Millions of people are using everyday online marketplace to find the right product or service, such as a room to rent for a short period of time, or a physical product to buy. As the marketplace grew, Data driven supply and demand suggestions became very important factors for the continued rapid growth and success of the marketplace. 

%One of the important aspects of host suggestion is pricing. Hosts need to be able to make data-driven decisions about how much they charge for their experience. However, since hosts do not access to data on how other hosts in their category are doing in the same market, or how seasonality effects the demand, it is up to Airbnb to leverage machine learning, statistics and optimization procedures on our data to provide hosts with the best price suggestion.

Online marketplace are used every day by millions of users worldwide to find and buy the right product or a service. Examples of online marketplaces include e-commerce sites where customers can buy products and travel sites where customers can book places to stay while they travel. The one important problem which all online marketplaces have in common is how to determine the right price for a product or a service. Depending on the type of marketplace, the price is usually set by the suppliers or the sellers. However, the marketplace typically has additional demand data and signals it could use to determine the right price of a product, i.e. its value. The marketplace can use it to provide services to suppliers such as price recommendations or even automatic price optimization where the price is changed dynamically to maximize suppliers revenue while exploiting the characteristics of the demand. On the other hand, the marketplace also leverages the signals about the inferred product value when surfacing products to buyers who are searching, i.e. in product recommendation and search ranking algorithms. The biggest challenge and the main topic of this paper is how to accurately determine the right price for a product. 

% Being able to determine the right price of a product or a service is of great significance for many downstream tasks that marketplaces are facing, such as price recommendation to suppliers, search ranking signals, as well as the ability to dynamically change the price of a product in order to maximize suppliers revenue while exploiting the characteristics of the demand.

%Online Marketplace are used everyday by users to find the right product or service, such as short term renting of a room, or a product to buy. As the marketplace grew, data driven supply and demand analysis became very important factors for the continued rapid growth and success of the marketplace. One of the important task for online marketplace services is to know the right price of a product or service. Price is usually set by the suppliers. Knowing the right price of a product or service is of great significance for many downstream tasks that marketplace are facing, such as price recommendation to suppliers, dynamically changing the price of a product in order to maximize suppliers revenue while exploiting the characteristics of the demand. 

Existing work on determining the best item price are typically considering items independently \cite{ye2018customized}, thus ignoring the multiple choice context user has when searching for an item. In the search setting, the user interacts with an algorithm that ranks the items to be shown to the user and the most relevant and most appealing results are shown first. Consequently, the top items receive substantial buyer attention and hence are a major contributor to purchases made in the marketplace. They are also in constant competition among each other and this needs to be taken into consideration when building the price inference model.

%\pv{Jiawei what are the current problems with existing methods? add one/two sentences here.} and without a real focus on where the purchase opportunity is coming from in an online marketplace. 

%When users are looking for a new product or service to buy they usually search for it. Query relevance model are used to rank the results and display N items to users. The top displayed items for each search sessions are therefore in "competition" among each other. These items are competing to be bought by users. On the other hand the items appearing on top results receive substantial attention by users, and hence make a major contribution to purchases. Therefore we need to consider this scenario for pricing purposes.   
%Items appearing on top results receive substantial attention by users, and hence make a major contribution to the purchases. 
%Therefore, it is essential to price listings from top sections properly so as to increase suppliers revenue

In this paper we propose a practical solution for determining the item prices that maximizes marketplace revenue based on search result data and the historical engagement with items that appeared in search. As already discussed, the search ranking algorithm aims at showing the most relevant items given the query parameters and ranking them in order of purchase probability. The set of shown results varies depending on entered query, keywords, filters used, etc. The items that rank higher receive substantially more attention than items that rank lower, and hence those items significantly contribute to whole marketplace revenue. Therefore, when aiming to increase marketplace and seller revenue it is essential to optimize and properly price the highly rank items. 

To determine the optimal price for an item, we consider a case in which there is only one buyer and many items, where each item is displayed alongside other items in search results. The item co-occurs in search with a different items every time, depending on the search query and therefore our predicted price for this item changes from search to search. For each item, their optimal price is a random variable depending on value distributions of co-displayed items. The final price for a given item is determined based on all the searches it appeared in and all the predicted prices it had in those searches, i.e. it is based on the distribution of predicted prices in all search results.

%Consumers are using everyday online marketplace search functionality to find the right product or service. These can be a room to rent for a short period of time, or a physical product to buy. 

%What is the right price to pay for a product or service? An important task for online marketplace services is to know the right price of a product. Knowing the right price of a product or service is of great significance for many downstream tasks that marketplace are facing, such as price recommendation to suppliers, dynamically changing the price of a product in order to maximize suppliers revenue while exploiting the characteristics of the demand. 
%Consumers are using everyday online marketplace search functionality to find the right product or service. These can be a room to rent for a short period of time, or a physical product to buy.  

%\pagebreak

We applied our price prediction model on Airbnb Experiences marketplace dataset. Airbnb Experiences are handcrafted activities designed and led by expert hosts that offer a unique taste of local scene and culture. The Experiences marketplace covers more than 1,000 destinations worldwide, including unique places like Easter Island, Tasmania, and Iceland. As the marketplace grew, data-driven host and guest recommendations became very important factors for the continued growth and success of the marketplace. One of the important aspects of hosting an experience is how to price it. Hosts determine the initial price of the activity they offer but they often change it afterwards to adapt to demand trends and competition. However, since hosts do not have a detailed overview of the demand, seasonality and the competition, they are not able to determine the optimal data-driven price themselves. For this reason Airbnb offers hosts price recommendations that are based on the up-to-date demand, seasonality and competition.

We used 6 months of anonymized search data from Airbnb Experiences marketplace to test our new pricing strategy, which aims to maximize revenue for the marketplace. We conducted comprehensive offline experiments to test the performance of the proposed strategy and demonstrated that it outperforms existing state-of-the-art pricing strategy based on several relevant metrics. 

%Hosts need to be able to make data-driven decisions about how much they charge for their experience. However, since hosts do not access to data on how other hosts in their category are doing in the same market, or how seasonality effects the demand, it is up to Airbnb to leverage machine learning, statistics and optimization procedures on our data to provide hosts with the best price suggestion.

% 2. Why is it interesting and important?

%When a consumer search for an experience, the display page varies based on different keywords (location and categories). Experience listings on top sections receive substantially more attention than bottom sections, and hence make a major contribution to the bookings and revenues. Therefore, it is essential to price listings from top sections properly so as to increase our revenue. In previous discussions, we develop a pricing model in the case when there is only one consumer and different items. For each item, its co-displayed items are changing in different search queries, and so do our predicted prices for this item. 
%For each item, their optimal price is a random variable depending on value distributions of co-displayed items. Rather than predicting the optimal price for the item using one search results, it makes more sense to estimate the distribution of the predicted price. 

% 3. Why is it hard? (i.e. Why do naive approaches fail?)

% 4. Why hasn't it been solved before?

% RELATED WORK
\section{Related work}
\label{sec:related}
%!TEX root = paper.tex
%

%Liu \emph{et al.} in \cite{Weibull2010DwellTime} ...
\begin{comment}
\pv{anything related to optimization, to pricing in general}

\pv{Puya, We need to see if there are other approaches to compare with.}

\pv{TBD by Jiawei: (1) write a section on optimal pricing we are considering and related work they are considering. (2) please start to write a section on value of previous eventual work on it.}
\end{comment}

%\pv{Puya is working on it}

Machine learning empowered pricing strategy has drawn remarkable attention in data science area, where the availability of large data enables the forecast of demand of products or other target quantities. There has been a rapid development in demand estimation using machine leaning \cite{bajari2015demand}. However, the generalization of these demand learning approaches to Airbnb marketplace is rather challenging. In \cite{ye2018customized}, the authors discussed several challenges of deriving an accurate demand estimation for Airbnb Homes, which are also applicable to Airbnb Experiences. For example, similarly to how the price of a specific Airbnb Home rarely changes for different days in the calendar year, the price of a specific Airbnb Experience hosted on different days is mostly the same or varies in small ranges, which makes the price extrapolation very difficult. In addition, experiences that belong to different categories are quite different, e.g. surfing vs. cooking class, which jeopardizes the generalization of the demand estimation from one category to another. Most pricing problems are applications of revenue management theory \cite{phillips2005pricing, chiang2007overview}, which has been an active topic in academic research and relates closely to dynamic pricing \cite{yeoman2010revenue, talluri2006theory, chen2015recent, den2015dynamic, dong2009dynamic, yabe2017robust, do2011demand, tsai2009dynamic}. Another branch focuses more on static pricing (see \cite{kunz2014demand} for a survey). However, most studies either assume a known demand function, or addresses the demand uncertainty through learning methods with limited practical applicability. Recently, \cite{ito2017optimization} considers a static pricing strategy for multi-products with substitution. Their work consists of a modeling stage to predict the demand given price and other features, and a second optimization step to find optimal price that maximize the profit function. 

Airbnb, as a growing community marketplace, gains considerable attention from academic research, ranging from its impact on traditional hotel industry \cite{zervas2017rise} to the analysis of its pricing strategies \cite{kwok2018pricing, chen2017consumer}. The problem of our interest is optimal item-pricing, which fits into the general framework of the optimal multi-dimensional deterministic mechanism design. The optimal mechanism design is a fundamental problem in economics and has attracted substantial attention in the theory of computation community. The major focus of existing literature in computer science is to study the computational efficiency of the mechanism \cite{chawla2007algorithmic, chawla2010multi, bhattacharya2010budget, hart2017approximate, cai2012algorithmic, cai2012optimal}. Among various problem formulations, we found the pricing scenario on Airbnb Platform is closely related to the Bayesian Unit-demand Item-Pricing Problem \cite{chawla2007algorithmic}, which studies the revenue-maximizing pricing strategy against a single unit-demand consumer of whom the valuation for all items $\bm{v} = (v_1, \cdots, v_N)$ are known by the seller. In addition, the consumer will select the item with the maximum value and price difference $v_i - p_i, i \in \{1, \cdots, N\}$. In solving this multi-dimensional mechanism,  \cite{chawla2007algorithmic} and \cite{chawla2010multi} obtain polynomial-time constant factor approximations to the optimal revenue. Later \cite{cai2015extreme} proposes a near-optimal polynomial-time approximation through algorithm reduction techniques such as probability and domain discretization. Our problem formulation follows that of \cite{cai2015extreme}, with an explicit characterization on the distributions of consumer's valuations (Normal distribution). When trying to implement their proposed algorithm, we realized that there was a practicality issue that hinders its real-world application. To address this issue, we build a machine learning model to learn the value distributions, and then resort to numerical algorithms to efficiently solve the problem.

% \pv{We need a little bit larger Related (half a column more), you can explore other marketplaces, Uber, Amazon, Pricing strategies. You should add differences with previous work on extreme value theorem, what they have done, and why this method is different? Explain in details}

\begin{comment}
\pv{TBD Puya and Mihajlo find papers on value}
read this - "Customized Regression Model for Airbnb Dynamic Pricing"
fast look at this - "Dynamic Pricing and Matching in Ride-Hailing Platforms"
\end{comment}

% MODEL
\section{Problem Definition}
\label{sec:probdef}
%!TEX root = paper.tex
%

In this section we give an overview of the problems we are interested in solving in order to find the best price for an item. 
Our first goal is to find the right price for each item appearing in the top search results such that we can maximize the revenue of the suppliers. In other words, in user's single search for items, we restricted our self to the top N ranked items on the page, $i=1, 2, ..., N$. For each item we assume a value distribution is given. This value distribution needs to be learned by us as well. We then want to find a price vector that will allow the suppliers to maximize the total revenue potential in that specific search. At last we will aggregate these pricing vectors over time in order to suggest the best item price to the item supplier. 

Formally, our pricing problem follows the scenario described in \cite{cai2015extreme}. Suppose there is a single seller, with $N$ products to sell, and one consumer who is unit-demand, i.e., the consumer is interested in purchasing at most one product. The seller has access to the distributions of the consumer's valuations on $N$ products $\bm{v}=(v_1, \cdots, v_N)$. Specifically, we assume that $v_i, i \in \{1,\cdots,N\}$ are mutually independent random variables drawn from a set of known distributions $F_i, i \in \{1,\cdots,N\}$. The consumer is assumed to purchase the product with the largest value and price gap. Then, given a price vector $\bm{p}=(p_1, \cdots, p_N)$, the expected revenue of the seller is defined as 
\begin{equation}
\label{objective}
\E R_{\bm{p}} = \sum_{i=1}^N p_i \cdot P_r[i=\argmax\{ v_j-p_j\} \wedge (v_i-p_i \geq 0)].
\end{equation} 
This problem formulation can be naturally applied to the Airbnb Experiences Marketplace. On Airbnb Experiences platform, there are thousands of travelers searching for experiences every day. For each search initialized by a user, our platform recommends top experiences based on a machine learning algorithm that takes as input various user features, query information and experiences features and outputs a probability of booking. Finally, the experiences are ranked based on the predicted booking probabilities and shown to the user in that order.

Our goal is to find a near-optimal price vector $\bm{p}$ that maximizes the expected revenue. In literature, a natural approach is to discretize the domain of price and then search for the optimum in the discretized domain. However, the running time of resulting algorithms is exponential in the number of products \cite{chawla2007algorithmic, hartline2005near}. \cite{cai2015extreme} develops a near-optimal polynomial-time algorithm for this problem, whose running time is polynomial in $\max\{n^{\log^{11} r\cdot\log\log r}, n^{\frac{\log^3r\cdot \log\frac{1}{\epsilon}}{\epsilon^8}}\}$ with $r\geq 1$ and $\epsilon$ being the approximation error. Nevertheless, even for moderate $\epsilon$, the resulting running time would still be very long. This severe trade-off between computational efficiency and approximation accuracy makes it impractical for us to use. In Section \ref{sec:rev_max}, we introduce how we adapt the formulation \eqref{sec:rev_max} to Airbnb Experiences marketplace to solve it efficiently.

Traditional revenue maximization pricing strategies usually optimize for the expected revenue defined by
\begin{equation}
\label{demand_rev}
   \E R_p^* = \sum_{i=1}^N p_i \cdot D(i, p_i),
\end{equation}
where $D(i, p_i)$ is the probability of the product $i$ being booked at price $p_i$. In common marketplaces, $D(i, p_i)$ is often predicted using product features, prices, and spatial and temporal data % and demand and supply signal
\cite{bajari2015demand}. Clearly, an accurate estimation of $D(i, p_i)$ is critical to the success of \eqref{demand_rev}, which is in general difficult for Airbnb Experiences marketplace. More importantly, formulation \eqref{objective} allows competition and substitution effects in demand as it replaces individual booking probability $D(i, p_i)$ with a winning probability. During a particular search, one of the driving factors of whether or not an experience will be booked is its relative competitiveness over the other co-displayed experiences, which is not well-captured by $D(i, p_i)$.

% MODEL
\section{Model}
\label{sec:model}

In this work, we propose a two-stage pricing model for supply revenue maximization using search events data. 
An accurate model for item values in the first phase and an efficient optimization in the second phase are the two key components of our pricing strategy.
In the first stage, we use a regression model to predict the booked price for each experience. The booked price is used as the value surrogate of an experience from the view of our experiences guests. In the second stage, we construct a supply revenue optimization problem on the basis of value model to find the optimal price in terms of revenue maximization for each search.
 To circumvent the inherent challenge of the exponential solution space and the inefficiency of existing discretization and approximation techniques \cite{cai2015extreme}, we apply numerical optimization algorithms to achieve practicality and scalability.  

\subsection{Value models}
\label{subsection:value}
Our first problem to solve is to find the inherent value of each item. We rely on the marketplace feedback at this stage, using the purchase event as confirmation of value for the booked experience. Specifically, whether to book an experience at a specific price or not is a decision made by our customers, and thus the booking event represents customer's validation of the price set by the seller. This is common practice for picking ground truth when modeling a marketplace, because it is a meeting point of demand and supply. We model the booked price using a variety of demand, supply and item relevant features. More formally we do a regression with the following loss function: 
\begin{equation}
\label{eqregression}
%\min_{\bm{\theta}} 
\sum_{j=1}^m (f_{\bm{\theta}}(\bm{x_j}) - y_j)^2 + \lambda \cdot \Vert \bm{{\theta}}\Vert^2,
\end{equation} 
where $y_j$ is the booked price, $\lambda$ is the regularization factor, $m$ is the number of bookings in the training set, $\bm{\theta}$ is the parameter to learn, and $\bm{x_j}$ is a set of features which describe the booked experience as well as the overall demand and market conditions.
\begin{table}[thbp]
%\vskip 0.15in
\centering
    \scalebox{1}{
        \begin{tabular}{l | l}
        \toprule
         Category & item category \\
         Host language & \# languages spoken by the host  \\
         Reviews & \# of reviews \\
         AVG Review & Average Review score  \\
         Photo Quality & The picture quality of the item \\
         Conversion Rate & Conversion rate of the item in search \\
         Demand Score & an index of demand for an item \\
        \bottomrule
        \end{tabular}
    }
    \vskip 0.05in
    \caption{A subset of features used to learn the value model. 
    }
    \label{tab:features}
\end{table}
Table \ref{tab:features} reports a subset of the features we considered during value learning phase.  
%We propose a solution based on demand and  supply match in the The main point where demand and supply are matching is 
%In this part, we model the booked price using a variety of market and experiences relevant features, and 

In order to find a value distribution for each item, we make an assumption that for each experience, the value follows a normal distribution $N(\mu_i, \sigma_i^2), i = 1,\cdots, N$, where $\mu_i$ is output from our value model and $\sigma_i^2$ is estimated by looking at values of the experience $i$ in the past month. The optimization for the objective function outputs price vectors that are in the same scale as input values, so using booked price ensures the price suggestion will land in a reasonable range of market prices. We use XGBoost \cite{chen2016xgboost} to train the value model. 

The output of the first phase is a predicted booking price $f_{\bm{\theta}}(\bm{x_j}) = v_j$ for every experience. We use this prediction as the mean of a value distribution for experiences in the second optimization stage. 

%In particula our experiment, for each date $ds$ and experience $i$, we use the sample standard deviation of values over $30$ days as an estimator of the $\sigma_i^2$. 

\subsection{Revenue Maximization Pricing Strategy}
\label{sec:rev_max}
So far we have proposed a method that is able to learn a value distribution for each item. In the next stage of our solution we aim at determining an optimal price for each item considering their search context and other items that appear alongside them in the search results. %This optimal price uses value distribution, 
Our proposed pricing strategy considers that these set of items are "competing" among each other, and its objective is to maximize the suppliers revenue. We proceed by computing the optimal price for each search event individually, and then aggregating the computed prices to output a single price suggestion for each item. 
%On top of the value model, we add another pricing strategy layer in order to maximize the revenue of the platform. 

%This strategy layer first computes the optimal price for each search session individually, and then aggregate the prices to output one single price suggestion for each experience. 
\begin{comment}
Figure \ref{fig:position_plot} is a distribution plot of the number of bookings wrt the ranking positions in search. For each search event, we look at experiences that are displayed at least top $20$ positions from paginated section, and compute the optimal price vector that maximize the expected revenue.
\end{comment}

More formally, for each search, we maximize the following objective function:
\begin{equation}
\label{new_objective}
\E R_{\bm{p}} = \sum_{i=1}^N p_i \cdot P_r[i=\argmax\{\alpha_i v_j-p_j\} \wedge (\alpha_iv_i-p_i \geq 0)],
\end{equation} 
where $\alpha_i$ is a search-specific value multiplier capturing information about user's preference in this search. For example, it could be the ranking score from the search ranking model. In general, if an item has a higher ranking score in a search, then its value in that search should also be amplified. This winning probability takes into account the fact that in addition to user preferences, price is an important determining factor during purchase. 
To compute the optimal price for each search, we rewrite the winning probability explicitly as a function of the distribution function of values. We assume that the values for experiences in the same search are mutually independent variables drawn from $N(\mu_i, \sigma_i^2)$, as described previously.  %where $\mu$ is output from our value model and $\sigma^2$ is estimated by looking at values of one experience in the past month.
To reduce the computation load, we constrain the search space to bounded sets:
\begin{itemize}

\item Truncate the value distributions $F_i, i=1,\cdots, N$ from $R$ to a bounded range $[v_{min}, v_{max}]$. In implementation, this means that for each value distribution, we shift all probability mass above $v_{max}$ to the point $v_{max}$ and all probability mass below $v_{min}$ to the point $v_{min}$. Choice of $v_{min}$ and $v_{max}$ will be given in Theorem \ref{thm1}. 
\item Restrict the price vector to a bounded set $[\xi v_{min}, v_{max}]$, $\xi > 1$ and consequently this reduces the search space from $R^N$ to bounded rectangles $[\xi v_{min}, v_{max}]^N$. This also ensures that price output from the algorithm will fall into a reasonable region.
%{\color{red}{Need to shift the probability mass to the integral limits.}}
%\item Discretize the support of the bounded distributions of values.  
\end{itemize}
These constraints on input and output variables incur loss on the revenue, and we will bound this loss in section \ref{theory}. After domain truncation, we can rewrite the winning probability of $i$-th item as: 
\begin{equation}
\label{prob_new}
\begin{aligned}
    &Pr(i=\argmax \{ \alpha_j v_j - p_j\} \wedge (\alpha_i v_i - p_i) \geq 0)\\
    &= Pr(\cap_{j\neq i}\{\alpha_j v_j - p_j < \alpha_i v_i - p_i\} \wedge (\alpha_i  v_i - p_i) \geq 0)\\
    &= \int_{v_{min}}^{v_{max}}Pr(\cap_{j\neq i}\{\alpha_j v_j - p_j < \alpha_i v - p_i\} \wedge (v \geq \frac{p_i}{\alpha_i})  \vert v_i = v)f_i(v)dv \\
    &= \int_{\max(\frac{p_i}{\alpha_i} , v_{min})}^{v_{max}}Pr(\cap_{j\neq i}\{\alpha_j v_j - p_j < \alpha_i  v - p_i\} \vert v_i = v)f_i(v)dv \\
    &=  \int_{\max(\frac{p_i}{\alpha_i} , v_{min})}^{v_{max}}\prod_{j\neq i}F_j((\alpha_i v - p_i+p_j)/\alpha_j )f_i(v)dv, \\ % & = \sum_{\max(p_i/\alpha_i , v_{min})}^{v_{max}}\prod_{j\neq i}F_j((\alpha_i v - p_i+p_j)/\alpha_j )f_i(v)dv \\
\end{aligned}
\end{equation}
where $F_i$ and $f_i$ are the distribution function and probability density function of the $i$-th experience's value distribution, respectively. In our experiments, \eqref{prob_new} is evaluated numerically by discretizing the value support.
The objective function can be rewritten as 
\begin{equation}
\E R_p = \sum_{i=1}^N p_i    \int_{\max(p_i/\alpha_i , v_{min})}^{v_{max}}\prod_{j\neq i}F_j((\alpha_i v - p_i+p_j)/\alpha_j )f_i(v)dv.
\end{equation}
For top $N$ experiences from each search event $j$, we calculate optimal price vector $p^{(j)}_1, \cdots p^{(j)}_N$. Since the expected revenue is a function of winning probabilities, which depend on the co-displayed experiences, the optimal price of the same experience is a random variable depending on the underlying value distributions of all top experiences in the search. For the experience $i$, we aggregate prices obtained from all search events where it appeared as one of the top results by taking the average, i.e., $p^*_i = \frac{1}{n_i}\sum_{j: i\in S_j}p^{(j)}_j$, where $S_j$ is the set of experiences that were ranked on top for the $j$-th search event, and $n_i$ is the number of search events where the experience $i$ were ranked on top.

\begin{comment}
\pv{TBD Jiawei, put here what has been done so far in terms of value building air/exp-pricing, the regression is the first model for value estimation}

\pv{TBD Jiawei,why the expected revenue as summation over probability of being bought x price or value is not good? lets add some details about it, and eventually we need to compare may be this.}

\pv{Next Steps for Jiawei. Things to focus on long term: (1) How to use value that is in a different scale as price? (2) Can we define the value differently? We could use similar regression but not normal distribution or totally different value. Any idea?  }

\pv{Does it make sense to consider the relative difference for instance in percentage? something like $i=\argmax (v_j-p_j)/v_j$ ?}

{\color{red}{I think this is a good idea.}}
\pv{everything here is going under model}

\end{comment}

\subsection{Theoretical results}
\label{theory}

%\pv{Please write all the theoretical ideas you have immediately down, do the proof you want to do, once we go to experiments you will forget about them} 

In this part, we study the revenue loss due to the restriction and truncation performed on the price and value distribution support, respectively. We first restate the Lemma 24 and Lemma 27 in \cite{cai2015extreme}, which present results on the restriction of price vector.
\begin{lemma}
\label{lm1}
Suppose that the values of items are independently distributed on $[v_{min}, v_{max}]$, and for any price vector $\bm{p} = (p_1,\cdots, p_n)$, construct a new price vector $\hat{\bm{p}}$ as follows: $\hat{p}_i = v_{max}, \ \text{if }  p_i > v_{max}$, $\hat{p}_i = v_{min}, \ \text{if }  p_i < v_{min}$, and otherwise $\hat{p}_i = p_i$. Then the expected revenue $ \E R_{\hat{\bm{p}}}$ and $ \E  R_{\bm{p}}$ from two price vectors $\hat{\bm{p}}$ and $\bm{p}$ satisfy $\E R_{\hat{\bm{p}}} \geq \E  R_{\bm{p}}$.
\end{lemma}

\begin{lemma}
\label{lm2}
$\forall \delta > 0$, for any price vector $\bm{p} = (p_1,\cdots, p_n)$, define $\bm{p}'$ as follows, let $p_i' = p_i$, if $p_i \geq \delta $, and otherwise $p_i' = \delta$.
Then the expected revenues $ \E  R_{\bm{p}}$ and $ \E  R_{\bm{p'}}$ from these two price vectors satisfy $\E R_{\bm{p'}} \geq \E R_{\bm{p}} - \delta$.
\end{lemma}
By Lemma \ref{lm1} and Lemma \ref{lm2}, we can see that when values of items are independently distributed on $[v_{min}, v_{max}]$, then for any price vector $\bm{p} \in R^N$, if we transform it to another price vector $\bm{p}''\in  [\xi v_{min}, v_{max}]^N$, $\xi > 1$, then the expected revenue $ \E R_{\bm{p}''}$ and $ \E  R_{\bm{p}}$ from these two price vectors $\bm{p}''$ and $\bm{p}$ satisfy $\E R_{\bm{p}''} \geq \E R_{\bm{p}} - \xi v_{min}$.

Next we show that we can truncate the support of value distributions to bounded range without hurting much revenue. 
\begin{theorem}
\label{thm1}
Given a collection of random variables $\{ v_i \}_{i=1,\cdots,N}$, where $v_i \sim N(\mu_i, \sigma_i^2)$, if we truncate their distributions to a bounded range $[v_{min}, v_{max}]$, where $v_{max} = \max_{i=1,\cdots, N}\{ Z_i^{\alpha}\}$ and $v_{min} = \min_{i=1,\cdots, N}\{ Z_i^{1-\alpha}\}$ with $Z_i^{\alpha}, \alpha \in (0,1)$ being the $\alpha$-quantile of distributions of $v_i$ \footnote{In our experiments, $\alpha$ is often set as $0.975$.}. For any price vector $\bm{p} \in [\xi v_{min}, v_{max}]^N$, $\xi > 1$, $\vert \E R_{\bm{p}} - \E \hat{R}_{\bm{p}} \vert \leq v_{max}\cdot (1-\alpha^N)$, where $R_{\bm{p}}$ and $\hat{R}_{\bm{p}}$ are the revenues when the consumer's values are distributed before and after truncation respectively. 
\end{theorem}
\begin{proof}
For a set of random variables $\{ v_i \}_{i=1,\cdots,n}$ that are distributed as $v_i \sim N(\mu_i, \sigma_i^2)$, define a new set of random variables $\{ \hat{v}_i \}_{i=1,\cdots,n}$ as 
\begin{equation}
    \hat{v}_i =  \left \{  
    \begin{aligned}
    &v_{max},  &\ if \  v_i > v_{max}\\
    &v_{min},  &\  if \ v_i \leq v_{min}\\
    &v_i, \ \ &\text{otherwise}.
    \end{aligned}
    \right.
\end{equation}
An important fact is that for any given price $\bm{p}$, $R_{\bm{p}}$ and $\hat{R}_{\bm{p}}$ are different only when $v_i \neq \hat{v}_i$ for some $i$, which reduces to the event that $\exists i, v_i > v_{max}$. To see this, if $\forall i$, $v_{min} < v_i \leq v_{max}$, then $v_i = \hat{v}_i$, and thus $R_{\bm{p}} = \hat{R}_{\bm{p}}$. If $\exists i$ such that $v_i \leq v_{min}$, and thus $\hat{v}_i = v_{min}$, then the price of item $i$ is higher than values $v_i$ and $\hat{v}_i$, so the item $i$ will not be purchased for both cases. Since the maximum price is $v_{max}$, we have the following bound for $\vert \E R_{\bm{p}} - \E\hat{R}_{\bm{p}} \vert$,
\begin{equation}
    \begin{aligned}
    \vert \E R_{\bm{p}} - \E\hat{R}_{\bm{p}} \vert &\leq v_{max} \cdot  \Pr[\exists i, v_i > v_{max}]\\
    & = v_{max} \cdot  \Pr[\max_i v_i > v_{max}]\\
    & = v_{max} \cdot (1-\Pr[\max_i v_i \leq v_{max}])\\
    & = v_{max} \cdot (1- \prod_{i=1}^{N}\Pr[ v_i \leq v_{max}])\\
%    &= v_{max} \cdot (1-\prod_{i=1}^{\infty}\Pr[v_i \leq Z_i^{\alpha}])\\
    & \leq v_{max} \cdot  (1- \alpha^N)
    \end{aligned}
\end{equation}
\end{proof}

\section{Experiments}
\label{sec:experiments}

In this section we aim to evaluate the effectiveness of the proposed solution using two search data sets, which are based on real users and not expert users. 
%For each search we consider the top 20 ranked experiences on user's search page
We use a time-based holdout of the search event data as a test set, for comparison purposes. %The general idea is to use the N results  %Given one user our general idea is given one user and one session, to see 
%use the suggested price coming   The general idea is to 
We proceed by describing our data set, experimental framework, metrics used for evaluation purposes and parameter tuning. We conclude by discussing the results.

%In this section we aim to evaluate the effectiveness of our
%query recommendation algorithm with respect to others by
%using a large query log
%https://www.overleaf.com/project/5cf826aedeb7c03cee17a204
%dataset
%Experimental and simulation Results are going here. Eventual Datasets.

%\pv{Jiawei, analysis on the data used for value, dataset description, results, tuning...}

%\pv{Jiawei, Results of the simulation of optimal pricing, where is not feasible, all the parameters, and plots or tables with description, ... and clear idea of why you choose to modify it in that way. }

\subsection{Data set}
Our data set is composed of search result data from Airbnb Experiences marketplace. As we mentioned in the introduction, this marketplace allows users to search for unique activities to do while travelling. A typical experience takes 2 hours on average. Experience Hosts can offer these activities to several travellers at the same time, and they can offer the same activity multiple times per day. We sampled a set of anonymized searches which occurred between 1st of January and 31st of May $2019$.  
We kept only searches in which the user actually purchased an experience. 
For each search we collected detailed information on displayed experiences, including the ranking position for all experiences in that search.
In Figure \ref{fig:position_plot} we show the ranking position of the experience that ends up being booked by the customer. As it can be observed, most of the booked experiences were ranking high on the search page. Therefore, we can conclude that the top ranking positions is where the biggest competition for bookings is happening. 
\begin{figure}[thb]
    \centering
    \includegraphics[width=1.0\linewidth]{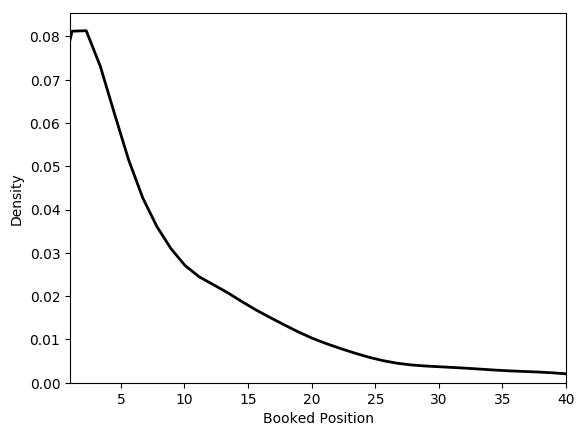}
    \caption{Booking ranking position distribution.} %\pv{can we re-plot this? x and y label please.}}
    \label{fig:position_plot}
\end{figure}
In addition to experience ranking information we collect a variety of other experience attributes that are used for feature engineering purpose, as described in Section \ref{sec:model}. 
Table \ref{tab:statistics} reports few more statistics of our data set %\footnote{We realize that this is a private data set, but we believe that the  }.
\begin{table}[thbp]
\centering
    \scalebox{1}{
        \begin{tabular}{l | r}
        \toprule
         Number of search events & $500k$ \\
         Number of distinct experiences & $54k$ \\
          Number of search events - Data set 2 (training) & $87k$\\
         Number of search events - Data set 2 (test) & $36k$\\
%        Number of distinct experiences - Data set 2 (training) & $36k$ \\
         Number of search events - Data set 1 (training) & $80k$\\
         Number of search events - Data set 1 (test) & $32k$\\
%         Number of distinct experiences - Data set 2 (training) & $32k$ \\
         %Test corpus - number of sessions & $x$ \\
         %Training corpus - number of sessions & $x$ \\
         %Average experience price & $x$ \\
         %... (value model features stats) & $x$ \\
        \bottomrule
        \end{tabular}
    }
    \vskip 0.05in
    \caption{Statistics of our data sets.
    }
    \label{tab:statistics}
\end{table}

%\item  Setting 1: $ds\_start = 2015/05/01, ds\_eval = 2015/05/16$, $ds\_end = 2015/05/20$. 
%\item Setting 2: $ds\_start = 2015/04/01, ds\_eval = 2015/04/23$, $ds\_end = 2015/04/29$. (In the training set, there are $32669$ experiences and $86819$ search sessions.)

A large portion of our training set is used for training the value model, which was described in Section \ref{subsection:value}. We use a small portion of the training set to train our pricing strategy which determines an optimal price for each experience. We consider two scenarios. In the first scenario we train our pricing strategy using the first $23$ days of April and we use the last week of April as our test set, referred to as Data set $1$. In the second scenario we train our pricing strategy on the first $23$ days of May and use the last week of May as a test set, referred to as Data set $2$.

\subsection{Experimental Framework}
For each day in training set, we have a set of searches, where each search contains many experiences displayed to the user. In order to find the optimal price for each search, we need to have the value distribution of each experience that appeared in that search. We therefore train a value model (\VALUE) using all the information available one day before the user search \footnote{Note that this corresponds to the real-world scenario where the machine learning models are re-trained everyday, and used for inference purpose during the next day.}.
For each search, we assume that the value follows a normal distribution centered at our predicted score for each experience. The scores are predicted using the machine learning model that was trained using past booked prices as labels, and listing attributes as features. It should be noted that no price related attributes were included as features for this model, and thus our predicted value is independent of the actual price of experiences.
In order to find the value distribution for each experience $i$ and $ds$, we use a sample standard deviation of values over $30$ days as an estimator of the $\sigma_i^2$.

When training the pricing strategy, we restrict ourselves to the top-$20$ experiences that appear on the search page in each search. 

%For each session, we assume that the value follows a normal distribution centered at our predicted score for each item. The scores are predicted using the past booked price as label, and listing features as features, and no price related features are included. Namely, our predicted score is independent of the actual price of listings.
%From $ds - 90$ to $ds$ days, we collect all useful sessions, and 
For each experience shown in a particular search, we use our pricing strategy to compute a price vector. Since the predicted price is search-specific, i.e., experience optimal price depends on the co-displayed experiences, we restrict our attention to the experiences that appeared frequently in the top $20$ search positions and combine the obtained prices from different searches for each experience by a simple average.

The revenue maximization (\REVMAX) strategy which was described in Section \ref{sec:rev_max} is applied to each search to determine the optimal price for each experience from top $20$ positions. The algorithm used to carry out the optimization is L-BFGS-B \cite{liu1989limited, zhu1997algorithm}. L-BFGS-B is a limited-memory quasi-Newton method for simple box-constrained optimization problem. In experiments, to reduce computational complexity, we may further restrict price search space of experience $i$ to a smaller range, e.g., $[\hat{v}_i - 2 \hat{\sigma}_i^2, \hat{v}_i + 2 \hat{\sigma}_i^2]$ where $\hat{v}_i$ and $\hat{\sigma}_i$ are estimated by the output from value model. At the end of this process, every experience has its actual price, and a suggested price. In cases where we are not able to provide any suggestion, e.g. experience were never ranked at top in searches, we use the predicted value from \VALUE model as the suggested price. In the next section we describe how we use the suggested price and the actual price together with booking information to build a set of metrics for evaluation purpose.
\subsection{Metrics}
\label{sec:metrics}
%\pv{Jiawei, plz start to improve from here, in general metrics is not bad but comparisons and results needs to be improved, I fixed many parts. It is not well written!}
%\pv{Why are you talking about baselines in metrics?}
% baselines to be moved
%\pv{Puya working on it}
In this section we introduce a set of metrics used in our offline evaluation. Most of our metrics are inspired by previous work on Airbnb dynamic pricing \cite{ye2018customized} and were adapted to our search-based methodology. The main assumption is that if an experience is booked after a particular search then the suggested price should have been same or higher than the booked price, otherwise we have ``regret", and if it was not booked, it would be better to suggest a lower price.

Let $i$ be a generic experience, and $S = \{S_j\}_j$ be a set of searches, where $S_j={i_{1j},  \dots, i_{N_j}}$ is a search containing $N_j$ experiences, $S^*_j$ denotes an experience which got booked during search $j$, $P_{ij}$ denotes the price of experience $i$ in search $S_j$ and $P_{i_{sugg}}$ denotes the suggested price for experience $i$ during the test search\footnote{Note that our suggestion for experience $i$ is fixed for the whole test set.}, we can define our metrics in the following way:

%We adapt some of the booking regret metrics described in \cite{ye2018customized} or search sessions  to search as follows,
\begin{itemize}
\item \textit{Booking Regret} (\emph{BR}), defined as,    
\begin{equation}
  BR = \median_{S_j \in S} \big( \max(\frac{P_{ij}-P_{i_{sugg}}}{P_{ij}}, 0),i = S^*_j \big),
\end{equation}
where we first compute the regret of each search as the relative difference between the booked experience price and the suggested price for that experience, and then we get the median over all searches. 
%over all sessions during which experience $i$ get booked. 
The intuition is that a good price suggestion method should not suggest a price that is lower than booked price, which hurts the revenue of suppliers. Thus a lower booking regret
is an indicator of a better price suggestion strategy. On the other hand, the lower the suggested price, the higher the regret w.r.t. the price that the experience was booked for;  
%will also improve the price suggestion adoption rate from the host. During our evaluation period, for booked experience instances from targeted search sessions, we compute the booking regret using all of them.

\item \textit{Weighted Booking Regret} (\emph{$BR_w$}), is defined as, 
\begin{equation}
  BR_w = \median_{S_j \in S} \big(  \max(P_{ij}-P_{i_{sugg}}, 0),i = S^*_j \big).
\end{equation}
Since booking regret captures the revenue loss w.r.t. the booked price but not the absolute loss of the suppliers, we define $BR_w$ to measure the absolute revenue loss;

\item \textit{Price Decrease Recall} (\emph{PDR}), is defined as, 
\begin{equation}
  PDR = \frac{\sum_{S_j \in S} |\{i \in S_j | i \neq S^*_j \land P_{i_{sugg}} < P_{ij} \}| }{\sum_{S_j \in S} |\{i \in S_j | i \neq S^*_j \}| },  
\end{equation}
where in the numerator we are considering the experiences that were not booked, and had a lower price suggested than their original price, and the denominator includes all the experiences that were not booked over all searches. The intuition here is that if the experience was not booked, and the price suggestion was higher than the actual price then we have a miss, otherwise we have a hit. Higher \PDR is a possible indication of a better price suggestion. However \PDR has limitations in the presence of competition, e.g. a properly priced experience may still not get booked when it co-occurs with experiences that are more competitive. Another point is that not all non-booked experiences during one search have to be sold for the best outcome. To overcome these limitations and have some insights on what each strategy is thinking when it lowers a price, we further defined \PDRHP (high revenue potential) and \PDRLP (low revenue potential), where we compute \PDR for the two subsets of non-booked experiences. In the first case (\PDRHP) we consider only experiences that have a value above the upper quartile of all experiences values, and a conversion rate that is below the lower quartile, despite receiving many impressions. In the second case (\PDRLP) we consider experiences that have a value below the lower quartile, and high conversion rate (above the upper quartile). A good pricing strategy should have \PDRHP that is higher than \PDRLP, indicating it is targeting the high revenue potential experiences.

%\jw{describe}
\iffalse
\item \textit{Price Decrease Recall High Potential} (\emph{PDR\_HP}), is defined as \emph{PDR} but we restrict only to experiences that have a high value (top $25\%$) but have a low number of bookings despite  appearing many times on top search results (least $25\%$ in conversion rate).
\fi 
%.but that are appearing many times on top positions without getting any booking. 
%at last positions of ranking (least $25\%$) and have a high value (top $25\%$). 
%This is considering that top ranked experiences have high probability of booking and therefore they don't need pricing improvement to get booked, while what we want is that high valued low ranked experiences get a higher probability of getting book due to price competitive advantage. 

\item \textit{Revenue Potential} (\emph{REV\_POTENT}), is defined as, 
\begin{equation}
\begin{aligned}
  REV\_POTENT &= 
  \frac{1}{|S|} \sum_{S_j \in S} \max_{\{i \in S_j | i \neq S^*_j  \land P_{i_{sugg}} < P_{ij} \}} gain_{ij} \cdot D_i ,  \\
 & gain_{ij} = P_{ij} - P_{S^*_{j}j},
\end{aligned}  
\end{equation}
where it considers all non-booked experiences for which we suggested a price that is lower than the actual price, and we want to approximately obtain what would have been the revenue gain if they were booked, due to the adoption of the suggested price. $D_i$ \footnote{We adjusted the demand by an elasticity of demand of $1.5$, that is an increase of $1.5\%$ for a price drop of $1\%$.} indicates a demand index, which is the probability of an experience getting booked. This will be described in more detail in the next section, where we learned this probability for  a comparison with other strategies. 
\item \textit{Recall} (\emph{RECALL}), is defined as the percentage of experiences for which the model was able to suggest a price. 

\end{itemize}

\subsection{Comparisons }
%\pv{Puya is working on it.}

In this section we describe the baselines and related work we used in our offline experiments for comparison to our proposed solution. The Customized Regression Model (\CRM ) proposed in \cite{ye2018customized} for determining an optimal price for Airbnb Home rental is the main related work in our comparisons. The \CRM method consists of two components: a booking probability model and a pricing strategy layer. The booking probability model is constructed to estimate the demand for a future night at specific prices. The authors recognize the difficulty in using the estimated demand directly in \eqref{demand_rev} to maximize the revenue, and therefore construct a second strategy layer that maps the booking probability to a price suggestion.
We implemented the \CRM booking probability model  \cite{ye2018customized} using the same set of features used in our value model (Table \ref{tab:features}), plus a pivot price. In contrast to Airbnb Home marketplace where only a single guest can book a single listing night, in the Airbnb Experience marketplace multiple guests can book the same Experience on the same day. Therefore, we needed to adapt the implementation of \CRM booking probability model to account for the difference by considering experiences which had at least a single booking as positives and ones which had zero bookings as negatives.

The second component of \CRM  requires to learn a demand index function $V_{\bm{\theta}}$, which takes the booking probability as input. To learn $\bm{\theta}$, the \CRM  strategy model adopts a customized loss function, and learns a $\bm{\theta}$ for each experience.  
%The main idea is to penalize price suggestions that fall out of an optimal price range. %$[L, U]$.
%Characterization of the optimal price range plays a critical role in the model performance, and should be align with major business metrics. 
%CRM model was trained on the listing level of homes. 
However, our data set has less price dynamics, and thus it was not ideal to learn $\bm{\theta}$ at the experience level. Therefore, we aggregated the experiences at the market and category level and learned one $\bm{\theta}$ for each market and category. When \CRM  is not able to suggest any price, we used the actual price as suggested one. %We than do price suggestion using \CRM , 

%is learned at a market category level. We make price suggestion using equation \eqref{home_price} for each experience instance with learned $\bm{\theta}$, and then average price suggestions for all instances to get one price suggestion at the experience level.

%The price suggestions are centered at $pivot\_price$ and adjusted by a learnable magnitude, i.e., 
\iffalse
\begin{equation}
\label{home_price}
P_{sugg} = pivot\_price \cdot V_{\theta}, 
\end{equation}
where $V_{\theta}$ is a nonlinear transformation of $P_r[pivot\_price]$ and $\theta$ is the parameter to learn.
%of a special functional form of $P_r[P_{pivot}]$, which is  

\fi
 %$P_r[P_{pivot}]$ is the booking probability estimated at price $P_{pivot}$ from the first model.
 
%$\theta_1$ and $\theta_2$ are learning parameters. $D$ is a normalized demand score, and in our experiment, we set $D$ as the ratio of bookings received over all available instances in the last $30$ days, and then normalize it on $[-1,1]$. $\phi_L < \phi_H$ are two demand constants chosen in $(1,2)$.

\begin{comment}
\begin{equation}
    \mathcal{L} = \argmin_{\bm{\theta}} \sum_{i=1}^N(L(P_i, y_i) - P_{sugg})^{+} + (P_{sugg} - U(P_i, y_i))^{+}.
\end{equation}
\end{comment}

\noindent\textbf{Baselines.} To better monitor the behavior of the set of metrics, we also compare with two baseline pricing strategies. The first strategy prices all products at zero (\ZERO), and the second strategy (\AVG) uses the average booked price observed in the training set as a suggested price.

% EXPERIMENTS
\section{Results}
\label{sec:results}

\begin{table*}[htbp] %the data here needs to be changed
	\begin{center}

		\scalebox{1}{
\begin{tabular}{lrrrrrrrrr}
\toprule
   position & Experience Title  &   \VALUE &  \REVMAX & price &  \CRM   \\%&  \AVG  \\
\hline
 $0$ & A Potter's Wheel in Brooklyn & $43.39$ &	$47.13$	& $45.00$	 & $32.08$ \\
 $1$ & Sailing Tour of New York and Brooklyn &     $83.45$ &  $81.31$ &   $80.00$ &  $164.45$  \\%80.98  \\
 $2$ & Rooftop yoga, massage and snacks &     $21.66$ &  $21.02$ &   $22.00$ &   $22.00$   \\%22.00   \\
 $3$ & \textbf{New NYC’s \#1 Rooftop Parties Tour} &     $40.94$ &  $40.13$ &   $35.00$ &   $34.01$ \\%  45.70   \\
 $4$ & Hasidic Brooklyn &     $47.73$ &  $50.40$ &   $49.00$ &   $37.38$ \\%  36.77  \\
 $5$ & Play \& cuddle with cats and kittens &     $15.82$ &  $13.67$ &   $14.00$ &   $14.77$ \\%  14.00  \\
 $6$ & Brooklyn Bridge photo-shoot &     $67.68$ &  $63.68$ &   $69.00$ &   $50.78$ \\%  $69.15$  \\
 $7$ & The Upper West Side Bookstore Crawl &     $15.00$ &  $13.14$ &   $20.00$ &   $7.07$ \\%   16.10 \\
 $8$ & Chinatown and Little Italy Tour &     $32.05$ &  $32.54$ &   $35.00$ &   $29.41$ \\%  35.48    \\
 $9$ & See 30+ Top New York Sights Fun Guide &     $25.27$ &  $29.68$ &    $51.00$ &   $33.91$ \\%  45.43    \\
% $10$ & Private Shoot w/Celebrity Photographer &     $91.03$ &  $90.19$ &   $89.00$ &   $69.50$ \\%  78.22   \\
%  & & & \vdots &  &  &\\
  $23$ & Taste of NYC Helicopter Tour  &    $129.20$ &  $122.98$ &  $129.00$ &  $121.65$ \\
\bottomrule
\end{tabular}
}		
    \vskip 0.1in
    \caption{Price suggestion for top $10$ ranking items plus the one with the highest predicted value (position $23$) in one single search. We report the ranking position, the experience title, the \VALUE model suggestion, \REVMAX price suggestion, the actual price of the experience, as well as the \CRM price suggestion. The experience at position $3$ is the booked one.}

%In addition Puya, if you have time can you make a search session with only $10$ experiences (not a subset of 20 like here)? That may be easy to explain the effect of our pricing strategy - low priority.}} 
    		\label{demo}
	\end{center}
	  \vspace{-0.5cm}
\end{table*}

In this section we present the experimental results which compare our proposed methodology to the baselines and related work.
Table \ref{demo} reports an example of top results appearing during one search event, with the title of the experience, the actual price shown to the user, and the ranking position. We report the results of the \VALUE model for each experience as well as \REVMAX and  \CRM suggested price for that experience. 
%$position$ is the ranking position in this search. $price$ refers to the listing price of each experience instance on that day. 
We can observe that different pricing strategies priced the experiences quite differently. In this example, the booked experience was ranked at position $3$, and both \VALUE model and \REVMAX strategy increased its calendar price, while  \CRM reduced its price. Compared to \VALUE, if a non-booked experience has a relatively high predicted value (e.g. helicopter tour), then the \REVMAX strategy tends to decrease its predicted value to improve its bookings, as this experience has a higher revenue potential.   

Figure \ref{fig:dist} plots the distribution of suggested prices from different strategies. We can observe that 
\REVMAX strategy has similar distribution to value model, both of them are more centralized and have lighter tails than the  \CRM and actual price.  \CRM has the trend to shift the prices to the left, which may increase bookings at the cost of worse booking regret and booking values.

\begin{figure}[htbp]
\centering
\includegraphics[width=1.0\linewidth]{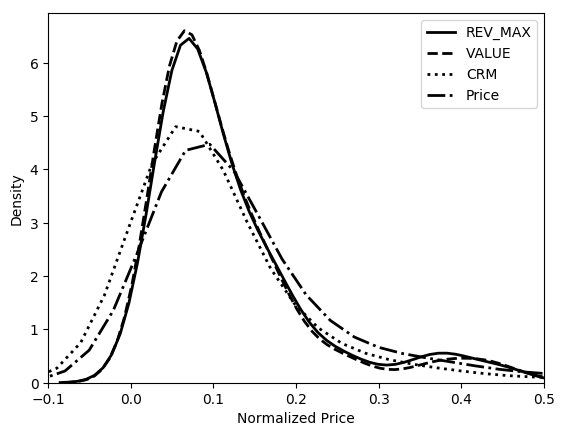}
\caption{Distributions of the suggested prices from different pricing strategy, and the actual price distribution.}
\label{fig:dist}
\end{figure}

%its $PDR$ in $low\_potent$ experiences is larger than that of $high\_potent$ experiences. \REVMAX strategy has higher $PDR$ in higher revenue potential group, which implies that our strategy reduces prices more aggressively for $high\_potent$ experiences. 

\begin{table}[htbp]
%\vskip 0.15in
\centering

    \scalebox{1}{

\begin{tabular}{lccccc}
\toprule
{} & \RECALL &  \BR &  \BRW &  \PDR  %&  PDR (high\_potent) &  PDR (low\_potent) %& \REVPOTENT 
&  \REVPOTENT \\ 
\midrule
\REVMAX &  $1.00$ &   $0.01$ &  $0.80$ &  $0.56$ % & $0.71$ &    $0.50$ % &  23.022 
&  $25.43$ \\
\VALUE &   $1.00$ &   $0.05$   & $1.91$ &  $0.69$ % &  $0.78$ &    $0.64$ %&   19.164
&  $21.25$ \\
\CRM &  $0.92$ &  $0.13$ &     $4.88$ &  $0.66$  %&    $0.52$ &  $0.64$ %&                 16.935 
&    $19.82$ \\
\ZERO &       $1.00$ &   $1.00$  &    $45.00$ &  $1.00$  % &      1.000 & 1.000 % & -15.715 
& $-39.29$ \\
\AVG &       \# &  $0.00$ &     $0.00$ &  $0.14$  % &   $0.08$ &    $0.14$  % & 13.048 
&   $14.30$ \\
\bottomrule
\end{tabular}
}
    \vskip 0.05in
    \caption{Evaluation results for New York for data set $1$. We report the results for all pricing models, and baselines. Our model outperforms \CRM and baselines in \REVPOTENT, with an almost zero \BR and a $1.0$ \RECALL.}
    \label{tab:metrics_nyu}
 \vspace{-0.3cm}
\end{table}

\iffalse
\begin{table}[htbp] %the data here needs to be changed
	\begin{center}
\centering
    \scalebox{1}{
\begin{tabular}{lrrrrr}
\toprule
{} & \RECALL &  \BR &  \BRW &  \PDR  %&  PDR (high\_potent) &  PDR (low\_potent) %& \REVPOTENT 
&  \REVPOTENT \\ 
\midrule
\REVMAX&     $1.00$ &   $0.00$ &     \textbf{$0.00$} &  \textbf{$0.57$} %&    $0.67$ &   $0.46$  %&  $24.31$
& \textbf{$28.94$} \\
\VALUE &       $1.00$ &   $0.01$ &     $0.11$ &  $0.62$ %&   0.808 &    0.455  % & 22.242 
& $26.08$ \\
\CRM &       $0.94$ &   $0.24$ &     $8.58$ &  $0.87$ %&   $0.81$ & $0.91$ %&14.499 
& $17.52$ \\
\ZERO &       $1.00$ &   $1.00$ &    $35.00$ &  $1.00$ %& 1.000 &    1.000  % & -25.075 
&              $-62.69$ \\
\AVG &       $1.00$ &   $0.00$ &     $0.00$ &  $0.15$ %& 0.038 &    0.182  %&  2.535 
&  $1.68$ \\
\bottomrule
\end{tabular}
		}
	\end{center}
	    \vskip 0.05in
    \caption{Evaluation results of San Francisco (data set 1)}
        \label{tab:metrics_sf}
 
\end{table}
\fi

Table \ref{tab:metrics_nyu} %and \ref{tab:metrics_sf} 
report results for New York's Airbnb Experiences market. Results suggest that our new pricing strategy is able to make suggestions for most of experiences in test set (> 99\%). Compared to  \CRM model, \REVMAX strategy decreases \BR and \BRW by $90\%$ and $83\%$, respectively. \REVMAX strategy also has a higher revenue potential than  \CRM. Compared with \VALUE model, \REVMAX improves \BR and \BRW by $78\%$ and $58\%$, respectively. Since there is a trade-off between \BR and \PDR, \REVMAX decreases \PDR of \VALUE model. In terms of revenue maximization, we attach more importance to lower \BR and \BRW than to a higher \PDR.
The baseline strategies, \ZERO and \AVG, perform very well in some metrics but fail for the others, which demonstrates a trade-off among these metrics and the importance of evaluating using the set of metrics in its entirety. In terms of \PDR, \CRM has higher \PDR than \REVMAX. %Note that \PDR is computed using all experiences from test set, in the presence of competitions and substitutions, a properly priced experience sometimes still not get booked when its co-displayed toghether with  other experiences that are more competitive. Therefore, a more informative way to use \PDR is to evaluate it on different segments of experiences. Specifically, we segment experiences into two groups based on their potential contribution to revenue boost defined as follows:
%for each experience $i$, define \emph{booking percentage} $b_i$ as the percentage of search sessions where it got booked over all search sessions where it appeared. An experience $i$ is said to have a high revenue potential, if its $b_i$ is below lower quartile of $\{b_i\}_i$, while its value $v_i$ is above the upper quartile of all experiences' values; and an experience $i$ is said to have a high revenue potential, if its $b_i$ is above upper quartile of $\{b_i\}_i$, and $v_i$ is below the lower quartile of all experiences' values. Intuitively, if an experience $i$ has low $b_i$ and high $v_i$, this implies that experience $i$ has a large room for improvement on revenue contribution, and thus by decreasing the price to improve its bookings, experience $i$ may potentially boost the revenue. On the contrary, if an experience $i$ has high $b_i$ and low $v_i$, its incremental contribution to the revenue is limited. Overall, a good pricing strategy should have higher \PDRHP than \PDRLP, which indicates it is targeting on the high revenue potenal experiences.

%We compute \PDR for experiences in $HP$ and $LP$ respectively, denoted as \PDRHP and \PDRLP. 
\begin{table}[htbp]
%\vskip 0.15in
\centering

    \scalebox{1}{

\begin{tabular}{lccccc}
\toprule
{}  &  \PDR &  \PDRHP &  \PDRLP  & \PDRHP/\PDRLP % $\frac{\PDRHP}{\PDRLP}$ %& \REVPOTENT 
\\ 
\midrule
\REVMAX  &  $0.56$  & $0.71$  &    $ 0.50$ &  $1.43$ % &  23.022 
 \\
\VALUE  &  $0.69$ &  \textbf{$0.78$} &    $0.64$  & $1.22$ %&   19.164
 \\
\CRM &  $0.66$ &    $0.52$ &  $0.64$ & $0.80$ %&                 16.935 
 \\
\ZERO &  $1.00$ &      $1.00$ & $1.00$  & $1.00$ % & -15.715 
 \\
\AVG &  $0.14$  &   $0.08$ &    $0.14$  & $0.55$ % & 13.048 
 \\
\bottomrule
\end{tabular}
}
    \vskip 0.05in
    \caption{\PDR distributions of New York for data set $1$. We report the results for all pricing models, and baselines. High \PDRHP and low \PDRLP is desirable. We report the ratio \PDRHP/\PDRLP, which outperforms \CRM and baselines.}
    \label{tab:pdr_nyu}
     \vspace{-0.3cm}
\end{table}

To overcome the limitation of \PDR as described in Section \ref{sec:metrics}, in Table \ref{tab:pdr_nyu} we report the result for New York market with the ratio between \PDRHP and \PDRLP (\PDRHP/\PDRLP). A higher ratio indicates a better pricing strategy. We can see that \REVMAX outperforms \CRM and baselines. \CRM has a higher overall \PDR, yet there is no big  difference between \PDRHP and \PDRLP; In contrast, \REVMAX and \VALUE have much higher \PDRHP than \PDRLP, and consequently result in a better \PDRHP than  \CRM. This implies that non booked experiences which have large room for improvement in terms of revenue contribution have had a lower price suggested, and thus by decreasing the price to improve their bookings, they have a high potential to boost the revenue.

%(\emph{$PDR_{HP}$}) is defined as, 
%\begin{equation}
%  PDR_{HP} = \frac{\sum\limits_{S_j \in S} |\{i \in S_j \land i \in HP | i \neq S^*_j \land P_{i_{sugg}} < P_{ij} \}| }{\sum\limits_{S_j \in S} |\{i \in S_j \land i \in HP| i \neq S^*_j \}| },  
% \end{equation}
%where we compute $PDR$ for experiences that in the high revenue potential group $HP$. To describe $HP$, first denote $B = \{b_i\}_i$ as the set of booking percentages, where $b_i$ = % $\frac{| \{j: i = S_j^*\} | }{|S| }$ 
%is the percentage of bookings of $i$ , and denote $A$ as the set of all calendar prices, then define $HP = \{i \vert b_i < Q_1(B) \land P_{i\cdot} > Q_3(A)\}$

%where on numerator we are considering the experiences that were not booked and in high revenue potential, and had a lower price suggested than their original price, and denominator includes all the experiences that were not booked over all sessions. The intuition here is that if an experience get not booked, and the price suggestion was higher than the actual price than we have a miss, otherwise we have a hit. A higher \emph{PDR} is the indicator of a better price suggestion algorithm. 

%\jw{Hi Puya, can we say these are our target markets, or it is confidential?} 
We conducted experiments in top $25$ Airbnb Experience markets and summarized the averaged results in Table \ref{metrics_2}\footnote{For confidentiality we removed the recall of the \AVG strategy.}. For Data set 1, result suggests that our new pricing strategy improves \BR by $48\%$ and \BRW by $44\%$ of \VALUE model, at a cost of decreasing the overall \PDR by $11\%$. Compared with  \CRM, \REVMAX lowers the \BR and \BRW by $90\%$ and $88\%$, respectively. To investigate \PDR, we report the ratio between \PDRHP and \PDRLP for our data sets in Table \ref{pdr_1}. We can observe that though \CRM has a higher overall \PDR, \REVMAX strategy and \VALUE are able to target more accurately the experiences that have a large revenue potential, i.e., larger \PDRHP compared to \PDRLP, and thus both have a higher \PDRHP/\PDRLP than that of \CRM.
In terms of revenue potential, our new strategy has comparable performance with the value model, and both outperform other strategies. The comparison results for Data set 2 are similar.
%In some market, \AVG returns null value for \REVPOTENT, and so we put a \# in the table).

\iffalse
\begin{table}[h] %the data here needs to be changed
	\begin{center}

		\scalebox{1}{
\begin{tabular}{lrrrrrrrr}
\toprule
{} &  \RECALL &   BR &  BR\_w &  PDR  %&  PDR (high\_potent) &  PDR (low\_potent) 
%& \REVPOTENT
&  \REVPOTENT\\
\midrule
\REVMAX &  $0.99$ &   $0.01$ & $0.55$ &  $0.50$  %& 0.641 &    0.388 %&   17.043
& $18.45$ \\
\VALUE &      $1.00$ &   $0.02$ &   $0.99$ &  $0.56$ %&   0.760 &    0.447 &   17.771 
&  $\textbf{19.47}$ \\
\CRM &     $0.93$ &   $0.13$ &     $4.73$  &  $0.63$ %&    0.557 &  0.540 %&  10.230
&               $11.89$ \\
\ZERO &        $1.00$ &   $1.00$ &    $43.30$ &  $1.00$ %&    1.000 &    1.000 %& -21.945 
&              $-54.86$ \\
\AVG &       \# &   $0.001$ &     $0.06$ &  $0.21$ %&    0.305 &    0.194 %& 7.806 
&   $7.97$ \\
\bottomrule
\end{tabular}	
}
	\end{center}
    \vskip 0.05in
	\caption{Evaluation Results of top $25$ Markets (data set 1)} 
	
		\label{metrics_1}
\end{table}
\fi

\begin{table}[htbp] %the data here needs to be changed
	\begin{center}

		\scalebox{1}{
\begin{tabular}{lrrrrrrrr}
\toprule
{} & \RECALL &  BR &  BR\_w &  \PDR &  \REVPOTENT \\ 
\midrule 
\multicolumn{6}{c}{Data set $1$}\\ 
\midrule
\REVMAX &  $0.99$ &   $0.01$ & $0.55$ &  $0.50$  %& 0.641 &    0.388 %&   17.043
& $18.45$ \\
\VALUE &      $1.00$ &   $0.02$ &   $0.99$ &  $0.56$ %&   0.760 &    0.447 &   17.771 
&  $19.47$ \\
\CRM &     $0.93$ &   $0.13$ &     $4.73$  &  $0.63$ %&    0.557 &  0.540 %&  10.230
&               $11.89$ \\
\ZERO &        $1.00$ &   $1.00$ &    $43.30$ &  $1.00$ %&    1.000 &    1.000 %& -21.945 
&              $-54.86$ \\
\AVG &       \# &   $0.001$ &     $0.06$ &  $0.21$ %&    0.305 &    0.194 %& 7.806 
&   $7.97$ \\
\midrule
\multicolumn{6}{c}{Data set $2$}\\
\midrule 
\REVMAX &       $1.00$ &   $0.01$ &    $0.52$ &  $0.49$ &   $16.62$ \\
\VALUE &       $1.00$ &   $0.02$ &     $0.66$ &  $0.52$ &   $17.67$ \\
\CRM &       $0.97$ &   $0.16$ &     $6.36$ &  $0.66$ & $10.64$ \\
\ZERO  &       $1.00$ &   $1.00$ &    $43.82$ &  $1.00$ &  $-51.00$ \\
\AVG &       \# &   $0.00$ &     $0.05$ &  $0.24$ &  \#  \\
\bottomrule
\end{tabular}
}
    \vskip 0.05in
    \caption{Evaluation results of top $25$ Markets using both data sets. We report the results for all pricing models, and baselines. Our \REVPOTENT is better than \CRM and baselines. Our \BR is almost zero, with a \RECALL almost $1.0$.} 
    		\label{metrics_2}
	\end{center}
 \vspace{-0.3cm}
\end{table}

\begin{table}[tbp] %the data here needs to be changed
	\begin{center}

		\scalebox{1}{
\begin{tabular}{lccccc}
\toprule
{} &  \PDR &  \PDRHP &  \PDRLP  & \PDRHP/\PDRLP% $\frac{\PDRHP}{\PDRLP}$
\\ 
\midrule
\multicolumn{5}{c}{Data set $1$}\\ 
\midrule
\REVMAX &  $0.50$  & $0.64$ &    $0.39$  & $1.65$  %&   17.043
 \\
\VALUE & $0.56$ &   \textbf{$0.76$} & $0.45$  &  $1.70$ %&   17.771 
 \\
\CRM  &  $0.63$ &    $0.56$ &  $0.54$ & $1.03$ %&  10.230
\\
\ZERO  &  $1.00$ &    $1.00$ &    $1.00$ & $1.00$  %& -21.945 
 \\
\AVG  &  $0.21$ &    $0.30$ &    $0.19$  & $1.57$ %& 7.806 
\\
\midrule
\multicolumn{5}{c}{Data set $2$}\\ 

\midrule
\REVMAX  &  0.49 &  0.59 &    0.37 &    1.59              \\
\VALUE&  0.52 &  0.70 &    0.41 &      1.68              \\
\CRM  &  0.66 &  0.63 &    0.57 &       1.10           \\
\ZERO  &  1.00 &   1.00 &    0.99 &     1.00            \\
\AVG &   0.24 &  0.21 &    0.23 &   0.92 \\
\bottomrule
\end{tabular}	
}
	\end{center}
    \vskip 0.05in
	\caption{\PDR distributions of top $25$ Markets for two data sets. We report the results for all pricing models, and baselines. High \PDRHP and low \PDRLP is desirable. We report the ratio \PDRHP/\PDRLP, which outperforms \CRM and baselines.} 
	
		\label{pdr_1}
	 \vspace{-0.2cm}	
\end{table}

\section{Conclusions and future work}
\label{sec:conclusion}
\balance
%!TEX root = paper.tex
%
In this paper, we propose a new pricing strategy that aims to maximize revenue for Airbnb Experiences Marketplace using search results. We started from a simple idea that when users are looking for items in a marketplace they typically search for it. The top ranked items shown to the user are therefore "competing" with each other for purchases. 

We found that the algorithmic solution proposed in previous work \cite{cai2015extreme} is not applicable to the real-world scenario. 
We proposed a practical solution to maximize the revenue for a unit-demand, single seller, multidimensional pricing problem in the context of search. Our unique input to the algorithm is a learned distribution of value for each item, and the set of top ranking items on search page. We then aggregate the best price for items found during search events and come up with the best price for each item. To reduce computation complexity of the proposed solution, we restrict the price space as well as the support of value distributions, and establish a theoretical bound on the revenue loss incurred by such restrictions. We conducted comprehensive offline evaluations to demonstrate the performance of the proposed strategy. Results using two real-world search data sets show that our pricing strategy outperforms related work. For future work, we would like to find more effective ways to aggregate optimal prices obtained from different search results, try different value models, and extend the work to other contexts. We also plan to conduct online experiments to further demonstrate the effectiveness of our pricing model.

%utilizing information such as clicks rate. We are also interested in extending our work to Airbnb Homes.

%Extensive experiments using a search datasets coming from Airbnb Experiences 

%Our solution is a two stage solutions, where we first learn a value model for all the items i

%We adapt from academic research on multidimensional unit-demand pricing and make it practical to use in our setting. 

%For future work, it would be of interest to find a more effective way to aggregate optimal prices obtained from different search sessions for each experience, utilizing information such as clicks rate. We are also interested in extending our work to Airbnb Homes.

%CONCLUSIONS
%\section{Appendix}
%\label{sec:appendix}
%\balance
%\input{appendix.tex}

\newpage
\bibliographystyle{acm}
\bibliography{paper}
\end{document}

\end{document}